\documentclass[manuscript]{acmart}

\usepackage{graphicx}
\usepackage{latexsym}
\usepackage{url}

\usepackage{amsmath}
\usepackage{amsthm}
\usepackage{booktabs}
\usepackage{algorithm}
\usepackage{algorithmic}

\usepackage{fontawesome}
\usepackage{tikz}
\usetikzlibrary{automata,positioning,decorations.markings,arrows,intersections,calc,shapes,shapes.symbols,shadings}

\setcopyright{rightsretained}
\copyrightyear{2023}
\acmYear{2023}
\acmDOI{}
\acmJournal{JACM}

\usepackage{macros}
\usetikzlibrary{patterns}

\usepackage{booktabs}

\begin{document}

\title{Omega-Regular Reward Machines}
\author{Ernst Moritz Hahn}
\affiliation{%
  \institution{University of Twente}
  \city{Twente}
  \country{The Netherlands}
}
\email{e.m.hahn@utewnte.nl}
\orcid{0000-0002-9348-7684}

\author{Mateo Perez}
\affiliation{%
  \institution{University of Colorado Boulder}
  \city{Boulder}
  \state{Colorado}
  \postcode{80309}
  \country{USA}
}
\email{Mateo.Perez@Colorado.EDU}
\orcid{0000-0003-4220-3212}

\author{Sven Schewe}
\affiliation{%
  \institution{University of Liverpool}
  \streetaddress{Ashton Building, Ashton Street}
  \city{Liverpool}
  \state{England}
  \postcode{L69 3BX}
  \country{UK}
}
\email{Sven.Schewe@liverpool.ac.uk}
\orcid{0000-0002-9093-9518}

\author{Fabio Somenzi}
\affiliation{%
  \institution{University of Colorado Boulder}
  \city{Boulder}
  \state{Colorado}
  \postcode{80309}
  \country{USA}
}
\email{Fabio@Colorado.EDU}
\orcid{0000-0002-2085-2003}

\author{Ashutosh Trivedi}
\affiliation{%
  \institution{University of Colorado Boulder}
  \city{Boulder}
  \state{Colorado}
  \postcode{80309}
  \country{USA}
}
\email{Asutosh.Trivedi@Colorado.EDU}
\orcid{0000-0001-9346-0126}

\author{Dominik Wojtczak}
\affiliation{%
  \institution{University of Liverpool}
  \streetaddress{Ashton Building, Ashton Street}
  \city{Liverpool}
  \state{England}
  \postcode{L69 3BX}
  \country{UK}
}
\email{D.Wojtczak@liverpool.ac.uk}
\orcid{0000-0001-5560-0546}

\renewcommand{\shortauthors}{Hahn, Perez, Schewe, Somenzi, Trivedi, and Wojtczak}

\begin{abstract}
Reinforcement learning (RL) is a powerful approach for training agents to perform tasks, but designing an appropriate reward mechanism is critical to its success. 
However, in many cases, the complexity of the learning objectives goes beyond the capabilities of the Markovian assumption, necessitating a more sophisticated reward mechanism.
Reward machines and $\omega$-regular languages are two formalisms used to express non-Markovian rewards for quantitative and qualitative objectives, respectively.
This paper introduces $\omega$-regular reward machines, which integrate reward machines with $\omega$-regular languages to enable an expressive and effective reward mechanism for RL. 
We present a model-free RL algorithm to compute $\varepsilon$-optimal strategies against $\omega$-regular reward machines and evaluate the effectiveness of the proposed algorithm through experiments.
\end{abstract}



\maketitle

\section{Introduction}
\label{sec:intro}
Reinforcement learning (RL)~\cite{Sutton18} is a powerful learning-based synthesis paradigm that relies on providing rewards and punishment signals to reinforce or diminish behaviours. 
This is based on the principle that behaviours that are repeatedly rewarded tend to become habitual, while behaviours that are punished tend to diminish with experience.
Therefore, translating high-level objectives into reward and punishment signals is critical to successful RL applications. 

Simple objectives, such as cost-optimal reachability or safety, can be intuitively encoded into Markovian reward signals. However, more complex objectives require a stateful reward mechanism. 
Two formalisms that have been used to express non-Markovian reward signals are formal specifications ($\omega$-regular languages and linear temporal logic)\cite{Sadigh14,Hahn19,camacho2019ltl,gaon2020reinforcement} and reward machines (Mealy machines based monitors with scalar rewards as outputs)\cite{RM1,camacho2018non}. 
The former is used to express long-run logical constraints, or ``qualitative'' specifications, while the latter is used for ``quantitative'' objectives, such as the discounted sum of rewards. Note that the terms \emph{qualitative} and \emph{quantitative} are used in this paper to differentiate between logic-based specifications over infinite behaviour and reward-based optimisation objectives. However, it's important to note that these terms can be misleading since maximising the probability of satisfying a logical property is technically a quantitative requirement. Similarly, finding a policy that provides a reward greater than a given budget can be considered a qualitative requirement.

This paper argues for the need to optimise quantitative rewards under logical constraints over infinite horizons and proposes a model to conveniently express such learning objectives, which we refer to as \emph{$\omega$-regular reward machines}. 
These models integrate the two formalisms, allowing for the optimisation of quantitative rewards while also enforcing logical constraints over infinite horizons.
The proposed $\omega$-regular reward machines provide a powerful and efficient way to specify complex reward structures for RL, enabling the effective and efficient training of RL agents.

\subsection{Reward Programming in RL}
Formal specifications, such as linear temporal logic (LTL), $\omega$-regular languages, and their generalisations~\cite{Baier08}, provide unambiguous and intuitive languages to express infinite-horizon requirements. 
However, manually designing rewards from higher-level specifications is tedious and error-prone. 
To address this challenge, researchers have proposed automatic translations from formal specifications to reward signals, providing a programmable, transparent, explainable, and trustworthy RL.

Sadigh et al.~\cite{Sadigh14} initiated the study of model-free RL, where learning objectives were expressed in LTL. 
They used LTL to $\omega$-automaton reduction~\cite{Baier08} to design a scalar reward signal, with the hope that maximising the discounted objective maximises the probability of satisfaction of the LTL objective. 
However, the work of Hahn \emph{et al.}~\cite{Hahn19} revealed challenges in translating formal specifications to reward machines, and proposed a correct translation from more general $\omega$-automata based requirements to reward machines. 
Since then, several formally correct reward schemes~\cite{Hahn20b,Hahn20c,Bozkur20,DBLP:journals/csysl/OuraSU20} have been proposed to automate 
$\omega$-regular reward translation.

Icarte et al.~\cite{RM1,icarte2022reward} advocated for the need of non-Markovian reward signals and popularised the use of Mealy machines to express such rewards. 
Reward machines provide an imperative language to program reward signals, allowing designers to better tune the reward logic by expressing their domain-specific expertise in the form of scalar rewards. However, we argue that---since reward machines encode a finite-horizon, albeit discounted, view of the environment---they fail to capture intuitive specifications and give rise to unintended and undesirable behaviours. 
To support this claim, we adapt the counterexample given by Hahn et al.~\cite{Hahn19} for the translation scheme of Sadigh et al.~\cite{Sadigh14} to show how reward machines fail to capture intuitive specifications.

\tikzset{
  strat/.pic={
    \fill [rotate around={#1:(0,0)}] (0.1,-0.1) -- (0.1,0.1) -- (0.6,0.1) --
    (0.6,0.3) -- (0.95,0) -- (0.6,-0.3) -- (0.6,-0.1) --cycle;
  }
}

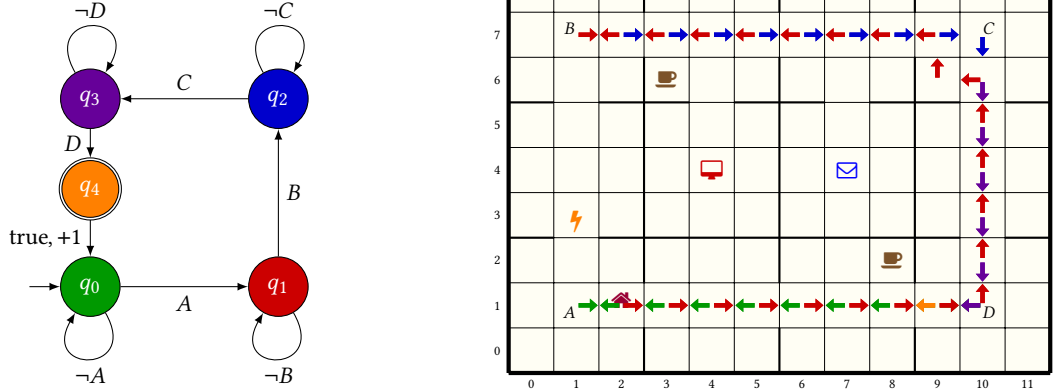
\begin{figure}
    \centering
    \begin{tikzpicture}[scale=1,transform shape]
        \node[state,initial,color=black,fill=green!60!black] (q0) {\textcolor{white}{$q_0$}};
        \node[state,color=black,fill=red!80!black] (q1) [right=2.5cm of q0] {\textcolor{white}{$q_1$}};
        \node[state,color=black,fill=blue!80!black] (q2) [above=2.5cm of q1] {\textcolor{white}{$q_2$}};
        \node[state,color=black,fill=violet!80!blue] (q3) [left=2.5cm of q2] {\textcolor{white}{$q_3$}};
        \node[state,color=black,fill=orange,accepting] (q4) [below=1.2cm of q3] {\textcolor{white}{$q_4$}};
        \path[->]
        (q0) edge [loop below] node [anchor=north] {$\neg A$} ()
        (q0) edge [] node [anchor=north] {$A$} (q1)
        (q1) edge [loop below] node [anchor=north] {$\neg B$} ()
        (q1) edge [] node [anchor=west] {$B$} (q2)
        (q2) edge [loop above] node [anchor=south] {$\neg C$} ()
        (q2) edge [] node [anchor=south] {$C$} (q3)
        (q3) edge [loop above] node [anchor=south] {$\neg D$} ()
        (q3) edge [] node [anchor=east] {$D$} (q4)
        (q4) edge [] node [anchor=east] {$\text{true},+1$} (q0)
        ;
    \end{tikzpicture}
    \hspace{2cm}
    \begin{tikzpicture}[scale=0.3,transform shape]
      \begin{scope}
        \edef\e{0}
        \edef\n{90}
        \edef\w{180}
        \edef\s{270}
    
        \fill[yellow!5] (0,0) rectangle (24,18);
        \foreach \x in {0,...,0} {
          \foreach \y in {0,...,11} {
            \path (1+24*\x+2*\y,-0.5) node {\Huge$\y$};
          }
          \draw[thick] (24*\x,6) -- ++(2,0) ++(2,0) -- ++(16,0) ++ (2,0) -- ++(2,0);
          \draw[thick] (24*\x,12) -- ++(2,0) ++(2,0) -- ++(4,0) ++ (2,0) -- ++(4,0) ++(2,0) -- ++(4,0) ++(2,0) -- ++(2,0);
          \foreach \y in {6,12,18} {
            \draw[thick] (24*\x+\y,0) -- ++(0,2) ++(0,2) -- ++(0,10) ++(0,2) -- ++(0,2);
          }
        }
        \foreach \y in {0,...,8} {
          \path (-0.5,1+2*\y) node {\Huge$\y$};
        }
        \foreach \x in {0,...,0} {
          \node[scale=1.6] (home-\x) at (5+24*\x,3+0.25)     {\Huge\textcolor{purple!80!black}{\faHome}};
          \node[scale=1.5] (coffee1-\x) at (7+24*\x,13) {\Huge\textcolor{brown!65!black}{\faCoffee}};
          \node[scale=1.5] (coffee2-\x) at (17+24*\x,5) {\Huge\textcolor{brown!65!black}{\faCoffee}};
          \node[scale=1.5] (envel-\x)   at (15+24*\x,9) {\Huge\textcolor{blue}{\faEnvelopeO}};
          \node[scale=1.5] (office-\x)  at (9+24*\x,9)  {\Huge\textcolor{red!80!black}{\faDesktop}};
          \node[scale=1.7] (bolt-\x) at (3+24*\x,6.7) {\Huge\textcolor{orange}{\faBolt}};
          \node[scale=1.3] (A-\x)       at (3-0.3+24*\x,3-0.3)  {\Huge$A$};
          \node[scale=1.3] (B-\x)       at (3-0.3+24*\x,15+0.3) {\Huge$B$};
          \node[scale=1.3] (C-\x)       at (21+0.3+24*\x,15+0.3){\Huge$C$};
          \node[scale=1.3] (D-\x)       at (21+0.3+24*\x,3-0.3) {\Huge$D$};
        }
        \draw[xstep=24cm,ystep=18cm,ultra thick] (0,0) grid (24,18);
        \draw[step=2cm,ultra thin] (0,0) grid (24,18);
        \pic[green!60!black] at (3,3) {strat=\e};
        \pic[green!60!black] at (3,3) {strat=\e};
        \pic[green!60!black] at (5,3) {strat=\w};
        \pic[green!60!black] at (5,3) {strat=\w};
        \pic[red!80!black] at (5,3) {strat=\e};
        \pic[red!80!black] at (5,3) {strat=\e};
        \pic[green!60!black] at (7,3) {strat=\w};
        \pic[green!60!black] at (7,3) {strat=\w};
        \pic[red!80!black] at (7,3) {strat=\e};
        \pic[red!80!black] at (7,3) {strat=\e};
        \pic[green!60!black] at (9,3) {strat=\w};
        \pic[green!60!black] at (9,3) {strat=\w};
        \pic[red!80!black] at (9,3) {strat=\e};
        \pic[red!80!black] at (9,3) {strat=\e};
        \pic[green!60!black] at (11,3) {strat=\w};
        \pic[green!60!black] at (11,3) {strat=\w};
        \pic[red!80!black] at (11,3) {strat=\e};
        \pic[red!80!black] at (11,3) {strat=\e};
        \pic[green!60!black] at (13,3) {strat=\w};
        \pic[green!60!black] at (13,3) {strat=\w};
        \pic[red!80!black] at (13,3) {strat=\e};
        \pic[red!80!black] at (13,3) {strat=\e};
        \pic[green!60!black] at (15,3) {strat=\w};
        \pic[green!60!black] at (15,3) {strat=\w};
        \pic[red!80!black] at (15,3) {strat=\e};
        \pic[red!80!black] at (15,3) {strat=\e};
        \pic[green!60!black] at (17,3) {strat=\w};
        \pic[green!60!black] at (17,3) {strat=\w};
        \pic[red!80!black] at (17,3) {strat=\e};
        \pic[red!80!black] at (17,3) {strat=\e};
        \pic[green!60!black] at (19,3) {strat=\e};
        \pic[green!60!black] at (19,3) {strat=\w};
        \pic[red!80!black] at (19,3) {strat=\e};
        \pic[red!80!black] at (19,3) {strat=\e};
        \pic[orange] at (19,3) {strat=\w};
        \pic[orange] at (19,3) {strat=\w};
        \pic[green!60!black] at (21,3) {strat=\w};
        \pic[green!60!black] at (21,3) {strat=\w};
        \pic[red!80!black] at (21,3) {strat=\n};
        \pic[red!80!black] at (21,3) {strat=\n};
        \pic[violet!80!blue] at (21,3) {strat=\w};
        \pic[violet!80!blue] at (21,3) {strat=\w};
        \pic[red!80!black] at (21,5) {strat=\n};
        \pic[red!80!black] at (21,5) {strat=\n};
        \pic[violet!80!blue] at (21,5) {strat=\s};
        \pic[violet!80!blue] at (21,5) {strat=\s};
        \pic[red!80!black] at (21,7) {strat=\n};
        \pic[red!80!black] at (21,7) {strat=\n};
        \pic[violet!80!blue] at (21,7) {strat=\s};
        \pic[violet!80!blue] at (21,7) {strat=\s};
        \pic[red!80!black] at (21,9) {strat=\n};
        \pic[red!80!black] at (21,9) {strat=\n};
        \pic[violet!80!blue] at (21,9) {strat=\s};
        \pic[violet!80!blue] at (21,9) {strat=\s};
        \pic[red!80!black] at (21,11) {strat=\n};
        \pic[red!80!black] at (21,11) {strat=\n};
        \pic[violet!80!blue] at (21,11) {strat=\s};
        \pic[violet!80!blue] at (21,11) {strat=\s};
        \pic[red!80!black] at (19,13) {strat=\n};
        \pic[red!80!black] at (19,13) {strat=\n};
        \pic[red!80!black] at (21,13) {strat=\w};
        \pic[red!80!black] at (21,13) {strat=\w};
        \pic[violet!80!blue] at (21,13) {strat=\s};
        \pic[violet!80!blue] at (21,13) {strat=\s};
        \pic[red!80!black] at (3,15) {strat=\e};
        \pic[red!80!black] at (3,15) {strat=\e};
        \pic[red!80!black] at (5,15) {strat=\w};
        \pic[red!80!black] at (5,15) {strat=\w};
        \pic[blue!80!black] at (5,15) {strat=\e};
        \pic[blue!80!black] at (5,15) {strat=\e};
        \pic[red!80!black] at (7,15) {strat=\w};
        \pic[red!80!black] at (7,15) {strat=\w};
        \pic[blue!80!black] at (7,15) {strat=\e};
        \pic[blue!80!black] at (7,15) {strat=\e};
        \pic[red!80!black] at (9,15) {strat=\w};
        \pic[red!80!black] at (9,15) {strat=\w};
        \pic[blue!80!black] at (9,15) {strat=\e};
        \pic[blue!80!black] at (9,15) {strat=\e};
        \pic[red!80!black] at (11,15) {strat=\w};
        \pic[red!80!black] at (11,15) {strat=\w};
        \pic[blue!80!black] at (11,15) {strat=\e};
        \pic[blue!80!black] at (11,15) {strat=\e};
        \pic[red!80!black] at (13,15) {strat=\w};
        \pic[red!80!black] at (13,15) {strat=\w};
        \pic[blue!80!black] at (13,15) {strat=\e};
        \pic[blue!80!black] at (13,15) {strat=\e};
        \pic[red!80!black] at (15,15) {strat=\w};
        \pic[red!80!black] at (15,15) {strat=\w};
        \pic[blue!80!black] at (15,15) {strat=\e};
        \pic[blue!80!black] at (15,15) {strat=\e};
        \pic[red!80!black] at (17,15) {strat=\w};
        \pic[red!80!black] at (17,15) {strat=\w};
        \pic[blue!80!black] at (17,15) {strat=\e};
        \pic[blue!80!black] at (17,15) {strat=\e};
        \pic[red!80!black] at (19,15) {strat=\w};
        \pic[red!80!black] at (19,15) {strat=\w};
        \pic[blue!80!black] at (19,15) {strat=\e};
        \pic[blue!80!black] at (19,15) {strat=\e};
        \pic[blue!80!black] at (21,15) {strat=\s};
        \pic[blue!80!black] at (21,15) {strat=\s};
      \end{scope}
    \end{tikzpicture}
  \caption{An $\omega$-regular reward machine and the corresponding learned strategy in the office patrol example. The arrow colour corresponds to the $\omega$-regular reward machine state colours. The strategy was learned with Q-learning, as described in Section~\ref{sec:experiments}. Rewards are zero unless otherwise specified.}
  \vspace{-2em}
  \label{fig:officegrid}
\end{figure}

\begin{example}[Counter-intuitive Office Grid-World.]
Figure~\ref{fig:officegrid} shows a grid-world example adapted from \cite{icarte2022reward}.  A robot patrols the four corner rooms of an office complex. 
However, unlike in \cite{icarte2022reward}, an electrical wire dangles from the ceiling along the path connecting rooms $A$ and $B$, blocking safe passage.
The robot can avoid the hazard by reaching $B$ from $A$ via $D$ and $C$ and then retracing its steps. 
Alternatively, it can try to fix the dangling wire but, in so doing, it may be damaged and put out of commission with probability $1/5$. 
If successful, the robot can then follow the shorter route that connects the corners in a simple cycle. 
Doling out reward every time the robot completes one round does not guarantee that the robot will follow the \emph{safe} strategy that avoids the dangling wire. This is because the \emph{risky} strategy, where the robot attempts to fix the dangling wire, incurs a risk that is offset by the reduction in path length, resulting in a higher expected discounted reward.  The key problem is that maximising the expectation of the cumulative reward is different from maximising the probability that the reward is positive. We transform the reward machine into an $\omega$-regular reward machine by marking the state $q_4$ as \emph{accepting} (or, equivalently, mark its outgoing transitions as accepting).
The overall objective is now modified so that visiting this accepting state infinitely with maximum probability takes precedence over the discounted reward. Using the technique described in this paper, we learn a strategy with Q-learning that satisfies the $\omega$-regular reward machine and patrols while avoiding the hazard.
\end{example}

At the same time, $\omega$-regular language-based RL is not sufficient to express simple quantitative preferences as shown in the following example.
\begin{example}[Office Grid-World with Preference.]
In the environment of Figure~\ref{fig:officegrid}, the robot may be tasked with picking up mail and coffee for the people who work in the ``office'' room.  The robot may be free to collect the mail before getting the coffee, or vice versa but, \emph{preferably} it should get mail first to prevent the coffee from getting cold.  This kind of preference is naturally expressed as rewards on the transitions of the reward machine, while the satisfaction of the main objective (delivery of mail and coffee to the office) is guaranteed by imposing an $\omega$-regular objective. 
\end{example}

While there are some efforts to combine quantitative and qualitative formalism to express learning objective, they have been limited in expressing $\omega$-regular languages~\cite{RM1,icarte2022reward,Bozkurt0P21}.
To the best of our knowledge, there is no prior work that can handle general $\omega$-regular objectives with discounted rewards. 
We propose $\omega$-regular reward machines to fill this gap.
The $\omega$-regular reward machines ($\omega$-RMs) are defined as nondeterministic B\"uchi automata equipped with a scalar reward function. 
In RL, $\omega$-RMs can act as interpreters that observe the sequence of actions taken by the learning agent and the corresponding sequence of observations from the environment and provide a sequence of scalar rewards.
Unlike reward machines~\cite{icarte2022reward}, $\omega$-RMs may be non-deterministic, and the resolution of these choices is delegated to the RL agent.
The goal of the RL agent is primarily to visit the accepting states infinitely often and then maximise the discounted sum of rewards.
Besides the examples above, there are multiple scenarios that call for expressing such combinations in RL.
\begin{itemize}
    \item {\bf Specification gaming.} 
    The term \emph{specification gaming} refers to the behaviour of a learning agent that satisfies the literal specification, often in the terms of the reward signal, but not the intended one. While it is impossible to eliminate instances of specification gaming beforehand, detecting such behaviour can provide clues to some underspecified constraints.

Although it is easy to explicitly express the constraints, designing reward signals that integrate such constraints can be challenging. For instance, consider the coastrunner example described in~\cite{Coastrunners}. The environment provides a positive reward for target hitting, assuming that the agent naturally wants to finish the boat race as soon as possible. However, since this assumption is not backed by any explicit reinforcement, the learned behavior may not align with the desired one.
One possible reward mechanism to address this issue is $\omega$-regular reward machines that provide the discounted sum of rewards predicated upon the satisfaction of the $\omega$-regular objective that the agent eventually terminates the boat race. Note that this requirement cannot be expressed directly as a reward machine or in formal logic.
    
    \item {\bf Relative Preference over Accepting States.} 
    B\"uchi automata~\cite{Baier08} generalise finite automata to accept infinite behaviours that cause the accepting transitions to be visited infinitely often. 
    $\omega$-regular reward machines generalise B\"uchi automata by allowing the designer to express relative preference over various accepting states.

    \item {\bf Repair Machines.}
    Another scenario where $\omega$-regular reward machines can be useful is in repair machines. Suppose we have an RL problem where the learning objective is expressed as a B\"uchi automaton, and the RL agent can rewrite some of the observations of the environment before they are evaluated by the interpreter. 
    The space of these repairs is given by a repair machine~\cite{DBLP:conf/atva/DaveKMT22}, which is defined as a weighted nondeterministic transducer, where the weight corresponds to the cost of the rewrite action. 
    In this case, the goal of the RL agent is to satisfy the objective given by the B\"uchi automaton while minimising a discounted sum of the costs associated with the repairs.
    The composition of the B\"uchi automaton-based specification and the repair machine can be expressed as an $\omega$-regular reward machines. By leveraging $\omega$-regular reward machines, we can ensure that the learning agent satisfies the intended specifications while optimising the cumulative reward and minimising the costs associated with the repairs.
    
    \item {\bf Ulysses Contract.} 
    Another application of $\omega$-regular reward machines is in modelling Ulysses contracts. A Ulysses contract is a decision made by an agent to restrict potentially tempting but irrational choices by a future version of itself. This form of self-binding contract is named after the Greek hero Ulysses who, in the Odyssey, has his crew tie him to the mast to safely enjoy the sirens' song.
    With $\omega$-regular RMs, the requirement to visit certain states infinitely often encodes the Ulysses contract, while the individual rewards encode various immediate rewards. This model allows the RL agent to maximise rewards without violating the specification. 
    By using $\omega$-regular reward machines to express Ulysses contracts, we can ensure that the learning agent follows a long-term plan of action that aligns with the desired objectives, even in the presence of potentially tempting but irrational choices. 
\end{itemize}

\subsection{Contributions}
The paper provides an expressive framework for designing RL agents that can satisfy complex temporal specifications while optimising the cumulative reward.
We introduce $\omega$-regular reward machines, which can express complex objectives involving both quantitative and qualitative aspects.
We then provide a convergent RL algorithm that approximates the optimal value for the $\omega$-regular objective.
In the case of a known model, we show the tractability of computing optimal value and $\varepsilon$-optimal policies. 
We also implement the proposed algorithm as an open-source tool and provide experimental results demonstrating its effectiveness.

\subsection{Related Work}
So far, we have cited several related works on reward machines~\cite{RM1,icarte2022reward} and formal specifications~\cite{Sadigh14,hasanbeig2019reinforcement,Hahn19,Hahn20b,DBLP:journals/csysl/OuraSU20,Bozkur20} in model-free reinforcement learning.
There has been substantial work on lexicographic objectives in optimisation and RL, including lexicographic discounted objectives~\cite{skalse2022lexicographic,chatterjee2006markov,Chatte07}, lexicographic $\omega$-regular objectives~\cite{Hahn21c}, and a combination of safety and discounted objectives~\cite{Bozkurt0P21}.
However, to the best of our knowledge, this is the first work to consider the general class of $\omega$-regular objectives with discounted rewards in model-free RL.


\section{Preliminaries}
\label{sec:prelims}
An alphabet $\Sigma$ is a finite set of letters. 
A finite string (resp. $\omega$-string)  over $\Sigma$ is defined as a finite sequence (resp.\ an infinite $\omega$-sequence) of letters from  $\Sigma$. 
We denote the empty string by $\varepsilon$.
We write $\Sigma^*$ and $\Sigma^\omega$ for the set of
finite and $\omega$-strings over $\Sigma$.  
A language (resp.\ $\omega$-language) $L$ over an alphabet $\Sigma$ is defined
as a set of finite strings (resp. $\omega$-strings).

\subsection{Markov Decision Processes}
Let $\DIST(S)$ denote the set of all discrete distributions over $S$.
\begin{definition}[Markov decision process]
 A Markov decision process (MDP) $\Mm$ is a tuple $(S, s_0, A, T, AP, L)$, where 
 \begin{itemize}
     \item $S$ is a finite set of states and $s_0 \in S$ is the initial state, 
     \item $A$ is a finite set of {\it actions},  \item $T\colon S \times A \to \DIST(S)$ is the {probabilistic transition function}, 
     \item $AP$ is the set of {\it atomic propositions} (observations), and 
     \item $L\colon S \to 2^{AP}$ is the {\it labelling function}. 
 \end{itemize}   
 For any state $s \in S$, we let $A(s)$ denote the set of actions that can be selected in
state $s$.
A {\it sub-MDP} of $\Mm$ is an MDP $\Mm' = (S', A', T', AP, L')$, where $S' \subset
S$, $A' \subseteq A$ is such that $A'(s) \subseteq A(s)$ for every $s \in S'$,
and $T'$ and $L'$ are analogous to $T$ and $L$ when restricted to $S'$ and
$A'$. Moreover $\Mm'$ is closed under probabilistic transitions.
An MDP is a Markov chain if $A(s)$ is singleton for all $s \in S$.
\end{definition}
 
For states $s, s' \in S$ and $a \in A(s)$, $T(s,
a)(s')$ equals $\Pr (s' | s, a)$.
A {\it run} of $\Mm$ is an $\omega$-word $\seq{s_0, a_1, s_1, \ldots} \in S
\times (A \times S)^\omega$ such that $\Pr(s_{i+1} | s_{i}, a_{i+1}) {>} 0$ for all $i
\geq 0$.
A finite run is a finite such sequence. 
For a {\it run} $r = \seq{s_0, a_1, s_1, \ldots}$ we define the corresponding
labelled run as $L(r) = \seq{L(s_0), L(s_1), \ldots} \in (2^{AP})^\omega$.
We write $\Runs^\Mm (\FRuns^\Mm)$  for the set of runs (finite runs) of the MDP
$\Mm$  and $\Runs{}^\Mm(s) (\FRuns{}^\Mm(s))$  for the set of runs (finite runs) of
the MDP $\Mm$ starting from the state $s$.  We write $\last(r)$ for the last state
of a finite run $r$.

A {\it strategy} in $\Mm$ is a function $\sigma \colon \FRuns \to \DIST(A)$ such that
$\supp(\sigma(r)) \subseteq A(\last(r))$, where $\supp(d)$ denotes the support
of the distribution $d$.
A memory skeleton is a tuple $M = (M, m_0, \alpha_u)$ where $M$ is a finite set of
memory states, $m_0$ is the initial state, and $\alpha_u: M \times \Sigma \to M$
is the memory update function.
We define the extended memory update function $\hat{\alpha}_u: M {\times} \Sigma^* \to M$ 
in a straightforward way.
A finite memory strategy for $\Mm$ over a memory skeleton $M$ is a Mealy machine 
$(M, \alpha_x)$ where $\alpha_x: S {\times} M \to \DIST(A)$ is the {\it next action function} 
that suggests the next action based on the MDP and memory state. 
The semantics of a finite memory strategy $(M, \alpha_x)$ is given as a strategy 
$\sigma: \FRuns \to \DIST(A)$ such that for every $r \in \FRuns$ we have that 
$\sigma(r) = \alpha_x(\last(r), \hat{\alpha}_u(m_0, L(r)))$.

A strategy $\sigma$ is {\it pure} if $\sigma(r)$ is a point
distribution for  all runs $r \in \FRuns^\Mm$ and is {\it mixed} (short for
strictly mixed) if $\supp(\sigma(r)) = A(\last(r))$ for  all runs
$r \in \FRuns^\Mm$.
Let $\Runs^\Mm_\sigma(s)$ denote the subset of runs $\Runs^\Mm(s)$ that
correspond to strategy $\sigma$ with initial state $s$.
Let $\Strat_\Mm$ be the set of all strategies.
We say that $\sigma$ is {\it stationary} if $\last(r) = \last(r')$ implies
$\sigma(r) = \sigma(r')$ for all finite runs $r, r' \in \FRuns^\Mm$.
A stationary strategy can be given as a function $\sigma: S \to \DIST(A)$.  
A strategy is {\it positional} if it is both pure and stationary.

 An MDP $\Mm$ under a strategy $\sigma$ results in a Markov chain $\Mm_\sigma$.
If $\sigma$ is a finite memory strategy, then $\Mm_\sigma$ is a finite-state
Markov chain.
The behaviour of an MDP $\Mm$ under a strategy $\sigma$ and starting state
$s \in S$ is defined on a probability space
$(\Runs^\Mm_\sigma(s), \Ff_{\Runs^\Mm_\sigma(s)}, {\Pr}^\Mm_\sigma(s))$
over the set of infinite runs of $\sigma$ with starting state $s$.  Given a random variable 
$f \colon \Runs^\Mm \to \Real$, we denote by $\eE^\Mm_\sigma(s) \set{f}$ the
expectation of $f$ over the runs of $\Mm$ originating at $s$ that
follow the strategy $\sigma$.

\subsection{Discounted Reward Objectives}
\label{sec:discounted}
The learning objective over MDPs in RL is often typically expressed using a Markovian reward function, i.e. a function $\rho\colon S \times A \times S \to \Real$ assigning utility to transitions.
A {\it rewardful} MDP is a tuple $\Mm = (S, s_0, A, T, \rho)$ where $S, s_0, A,$ and $T$ are defined in a similar way as for MDP, and $\rho$ is a Markovian reward function.
A rewardful MDP $\Mm$ under a
strategy $\sigma$ determines a sequence of random rewards
${\rho(X_{i-1}, Y_i, X_i)}_{i \geq 1}$, where $X_i$ and $Y_i$ are the
random variables denoting the $i$-th state and action, respectively.
For $\lambda \in [0, 1[$, the {\it discounted reward} 
$\EDisct(\lambda)^\Mm_\sigma(s)$ is defined as
$\lim_{N \to \infty} \eE^\Mm_\sigma(s) \Big\{\sum_{1 \leq i \leq N}
  \lambda^{i-1} \rho(X_{i-1}, Y_i, X_i)\Big\}$.
We define the optimal discounted reward
$\EDisct^\Mm_*(s)$ for a state $s \in S$ as 
$\EDisct^\Mm_*(s) \rmdef \sup_{\sigma \in \Strat_\Mm} \EDisct^\Mm_\sigma(s)$.
A strategy $\sigma$ is discount-optimal if
$\EDisct^\Mm_\sigma(s) = \EDisct^\Mm_*(s)$ for all $s {\in} S$.

Often, complex learning objectives cannot be expressed using Markovian reward signals. 
A recent trend is to express learning objectives using finite-state reward
machines~\cite{icarte2022reward}. 
A (nondeterministic) reward machine is a tuple $\Rr = (\Sigma, U, u_0, \delta, \rho)$
where $U$ is a finite set of states, $u_0 {\in} U$ is the starting state,
$\delta\colon U {\times} \Sigma \to 2^U$ is the transition relation, 
and $\rho\colon U \times \Sigma \times U \to \Real$ is the reward function.

Given an MDP $\Mm = (S, s_0, A, T, AP, L)$ and a reward machine $\Rr = (2^{AP}, U, u_0, \delta, \rho)$,  their product 
$\Mm{\times}\Rr = (S{\times} U, (s_0,u_0), (A {\times} U),
T^\times, \rho^\times)$
is a rewardful MDP where
$T^\times\colon (S {\times} U) \times (A {\times} U) \to \DIST(S{\times} U)$ is such that
\[
((s,u), (a, u'))(({s}',{u}')) \mapsto 
\begin{cases}
T(s,a)({s}') & \text{if } u' {\in} \delta(u,L(s)) \\
0 & \text{otherwise.}
\end{cases}
\]
and $\rho^\times\colon (S{\times} U) \times (A {\times} U) \times (S{\times} U)\to \Real$ is defined such that 
$\rho^\times((s,u), (a, u'), (s', u'))$ equals $\rho(u, L(s), u')$ if $(u,L(s),{u}') \in \delta$.
For discounted reward objective, the optimal strategy of
$\Mm{\times}\Rr$ are positional on $\Mm{\times}\Rr$.
Moreover, these positional strategies characterise a finite memory strategy (with memory skeleton based on the  states of $\Rr$ and the next-action function based on the positional strategy) over $\Mm$ maximising the learning objective given by $\Rr$. 

In our reductions, we make use of total reward objective $\ETotal^\Mm_*(s)$ defined in a similar fashion as the discounted objective when the discount factor $\lambda$ is equal to $1$.
The concepts of expected total reward and optimal strategy is defined in an analogous manner.

\subsection{Omega-Regular Specifications}
\label{sec:omega}
\begin{definition}[B\"uchi automaton]
    A \emph{B\"uchi automaton} is a tuple
${\mathcal A} = (\Sigma,Q,q_0,\delta,F)$, where \begin{itemize}
    \item $\Sigma$ is a finite
\emph{alphabet}, 
\item $Q$ is a finite set of \emph{states}, 
\item $q_0 \in Q$ is
the \emph{initial state}, 
\item $\delta \colon Q \times \Sigma \to 2^Q$ is
the \emph{transition function}, and 
\item $F \subseteq Q \times \Sigma\times Q$ is the
set of \emph{accepting transitions}.
\end{itemize}

\end{definition}
A \emph{run} $r$ of ${\mathcal A}$ on $w \in \Sigma^\omega$
is an $\omega$-word $r_0, w_0, r_1, w_1, \ldots$ in
$(Q \times \Sigma)^\omega$ such that $r_0 {=} q_0$ and, for $i > 0$,
$r_i \in \delta(r_{i-1},w_{i-1})$.  Each triple
$(r_{i-1},w_{i-1},r_i)$ is a \emph{transition} of ${\mathcal A}$.
We write $\infi(r)$ for the set of transitions that appear infinitely
often in the run $r$.
A run $r$ of ${\mathcal A}$ is \emph{accepting} if $\infi(r) \cap
F \neq \emptyset$. 
The \emph{language} $\Ll(\Aa)$ of ${\mathcal A}$ is the subset of words in
$\Sigma^\omega$ that have accepting runs in ${\mathcal A}$.
A language is $\omega$-\emph{regular} if it is accepted
by a B\"uchi automaton.

Given an MDP $\Mm = ( S, s_0, A, T, AP, L )$
and a B\"uchi automaton $\mathcal{A} = (2^{AP}, Q, q_0, \delta, F )$,
their \emph{product}
$\Mm \times \mathcal{A} = ( S {\times} Q, (s_0,q_0), A {\times} Q, T^\times, F^\times )$ is an MDP with accepting transitions $F^\times$ where
$T^\times\colon (S {\times} Q) \times (A {\times} Q) \to \DIST(S \times Q)$ is such that
\[
((s,q),(a,q'))(({s}',{q}')) \mapsto 
\begin{cases}
T(s,a)({s}') & \text{if } (q,L(s,a,{s}'),{q}') {\in} \delta \\
0 & \text{otherwise.}
\end{cases}
\]
The set of accepting transitions 
$F^\times \subseteq (S \times Q) \times (A \times Q) \times (S
\times Q)$ is defined by $((s,q),(a,q'),(s',q')) \in F^\times$
if, and only if, $(q,L(s,a,s'),q') \in F$ and $T(s,a)(s') > 0$.

A strategy $\sigma$ on the product defines a strategy $\sigma'$ on the MDP with the same value, and vice versa.  Note that for a stationary
$\sigma$ on the product, the strategy $\sigma$ on the MDP may need memory.

An {\it end-component}~\cite{Luca98} of an MDP $\Mm$ is a sub-MDP $\Mm'$ such that for every state pair $s, s' \in S'$ there is a strategy that can reach $s'$ from $s$ with positive probability. 
A maximal end-component is an end-component that is maximal under set-inclusion.
Every state $s$ of an MDP $\Mm$ belongs to at most one maximal end-component.
End-components and runs of the product MDP are defined just like for MDPs.
An accepting end-component is an end-component that contains an accepting transition.
A run of $\Mm {\times} \mathcal{A}$ is accepting if
$\inf(r) \cap F^\times \neq \emptyset$.
We define the B\"{u}chi satisfaction probability 
$\blike_{\sigma}(s)$
of a strategy $\sigma$ as the probability of this strategy generating an accepting run, i.e.
\[
\textrm{Pr}_{\sigma}^{\Mm\times\Aa} \Set{ r \in
\Runs_{\sigma}^{\Mm\times\Aa}(s,q_0) : \inf(r) \cap F^\times
\neq \emptyset } .
\]
Similarly, $\blike(s)$ is the optimal satisfaction probability over the product, i.e.\  
$\blike(s) = \sup_{\sigma}\blike_{\sigma}(s,q_0)$.
We say that a strategy $\sigma_*$ is \emph{B\"uchi-optimal} from $s \in S$ if $\blike(s) = \blike_{{\sigma_*}}(s, q_0)$.

\section{Omega-Regular Reward Machines}
\label{sec:omega_rm}

Our definition of $\omega$-regular reward machine integrates the definitions of reward machines and B\"uchi automata. 
The notion of product with an MDP is defined in a similar fashion.
The optimisation objective is to compute optimal discounted reward over all B\"uchi-optimal strategies and near-optimal strategies achieving this reward.

\begin{definition}[$\omega$-Regular Reward Machine ($\omega$-RM)]
An $\omega$-RM is a tuple $\Rr = (\Sigma, U, u_0, \delta, \rho, F)$ where 
\begin{itemize}
    \item $U$ is a finite set of states, 
    \item $u_0 \in U$ is the starting state, \item $\delta\colon U \times \Sigma \to 2^U$ is the transition relation, 
    \item $\rho\colon U \times \Sigma \times U \to \Real$ is the reward function, and 
    \item $F \subseteq U \times \Sigma \times U$ is the
set of \emph{accepting transitions}.
\end{itemize}
\end{definition} 
Given an MDP $\Mm = ( S, s_0, A, T, AP, L )$
and an automaton $\mathcal{R} = (2^{AP}, U, u_0, \delta, \rho, F)$,
their \emph{product}
$\mathcal P = \Mm \times \mathcal{R} = ( S {\times} U, (s_0, u_0), A {\times} U, T^\times, \rho^\times, F^\times )$ is an MDP with initial
state $(s_0,u_0)$ and accepting transitions $F^\times$ where
$T^\times\colon (S {\times} U) \times ((A {\times} U) )\to \DIST(S {\times} U)$ is such that:
\[
T^\times((s,u), (a, u'))(({s}',{u}')) = 
\begin{cases}
T(s,a)({s}') & \text{if } u' {\in} \delta(u,L(s)) \\
0 & \text{otherwise.}
\end{cases}
\]

The set of accepting transitions 
$F^\times \subseteq (S \times U) \times (A \times U) \times (S \times U)$ is defined by $((s,q),(a,q'),(s',q')) \in F^\times$
if, and only if, $(q,L(s),q') \in F$ and $T(s,a)(s') > 0$.

Let us fix the product MDP $\mathcal P = \Mm \times \mathcal{R}$ as the tuple $(Q, q_0, A, T,  \rho, F)$ for the rest of this section.
To define the optimisation objective for $\omega$-RM, we need to define the following concepts over the product MDP $\mathcal P$.
\begin{itemize}
    \item The \emph{B\"uchi satisfaction probability}
    $\blike_\sigma (s)$ is the probability to satisfy the B\"uchi objective by a strategy $\sigma$ from a given state $s$ and is defined similar to that for B\"uchi automata (Section~\ref{sec:omega}). 
    The \emph{optimal satisfaction probability} $\blike(s)$ and a B\"uchi-optimal strategy $\sigma^*$ that achieves these, i.e., $\blike(s_0) = \blike_{\sigma^*}(s_0)$ is defined similar to that for B\"uchi automata (Section~\ref{sec:omega}).
    \item
    The \emph{optimal B\"uchi-discounted value} $\bval$ is the optimal  discounted reward among B\"uchi optimal strategies, i.e.,
    \[
    \bval : q \mapsto \sup_\sigma \Set{\val_\sigma(q) \mid \blike(q)=\blike_\sigma(q)} 
    \]
    where $\val_\sigma : q \mapsto  \EDisct(\lambda)^\Mm_\sigma(q)$ is the discounted value of $\sigma$.
    \item An \emph{optimal B\"uchi-discounted strategy} is a strategy $\sigma$ that attains optimal B\"uchi-discounted value, i.e.,  $\blike(q_0)=\blike_\sigma(q_0)$ and, for a given $\varepsilon>0$ an $\varepsilon$-optimal (near optimal) B\"uchi-discounted strategy is such that $\bval_\sigma(q_0)>\val_\sigma(q_0)-\varepsilon$. 
\end{itemize}
We seek near-optimal strategies that maximise the chance of satisfying the B\"uchi objective, but will only be arbitrarily close to satisfying the discounted reward objective due to the following observation. 
\begin{lemma}
\label{lemma:noopt}
The optimal B\"uchi-discounted strategies may not exist.
\end{lemma}
\begin{proof}
Consider an MDP where one can freely choose the next letter from an alphabet $\{a,b\}$ and have a reward of $1$ for $a$ and $0$ for $b$, as well as a primary B\"uchi objective to see infinitely many $b$'s, then we cannot achieve an expected reward of $\frac{1}{1-\lambda}$ while satisfying the B\"uchi objective. We can, however, get arbitrarily close, e.g., by producing $a$'s until a reward $>\frac{1}{1-\lambda}{-}\varepsilon$ is collected for any given $\varepsilon > 0$, and henceforth produce $b$'s.
While the optimal B\"uchi-discounted value is $\frac{1}{1-\lambda}$, no (finite or infinite memory) strategy can attain it.
\end{proof}

\subsection{Known MDP: Probabilistic Model Checking}
Consider the problem to compute the optimal B\"uchi-discounted value and near-optimal strategies when $\mathcal M$ and $\mathcal R$, and therefore their product $\mathcal P$, are known. 
A possible first step is to model check the MDP $\mathcal P$ against the B\"uchi objective. This provides the probability of achieving the B\"uchi objective from every state together with a positional strategy $\sigma^*$ of how to achieve it.

Model checking the product MDP $\mathcal P = (Q, q_0, A, T, \rho, F)$ is a standard operation~\cite{Baier08,Luca98}. One would typically start with qualitative model checking, which consists of two intertwined procedures:
\begin{enumerate}
    \item Remove all states in $Q$ from which no accepting transition is reachable with positive probability. If any states were removed, go to step 2.
\item Recursively remove state-action pairs $(q,a)$ where $T(q,a)(q'){>}0$ for any state $q'$ that has been previously removed, and remove states $q$ such that all of its state-action pairs $(q,a)$ have already been removed. If any state was removed at the end of this procedure, go back to step 1.
\end{enumerate}
Both steps work in time linear in the transition graph of $\mathcal P$, and a fixed point is reached in at most $|Q|$ steps, because a new procedure call is made only if at least one  state was removed.
The remaining states, $Q_1 \subseteq Q$, are those, for which we can satisfy the B\"uchi objective almost surely, and the last application of (1) provides such a strategy.

To extend this method to quantitative model checking (computing the optimal probability of satisfying the B\"uchi objective), we can simply add, for all states $q$ removed during the procedure above (whose set will be denoted by $Q_{<1} = Q\setminus Q_1$) and all actions $a$, variables $p_{(q,a)}$ and $p_q$ that represent the probability to win when taking the state-action pair $(q,a)$ and when starting at $q$, respectively.
To calculate the correct probabilities, we define the following linear program that these probabilities have to satisfy.
\begin{mdframed}[backgroundcolor=black!10, roundcorner=10pt,leftmargin=1, rightmargin=1, innerleftmargin=15, innertopmargin=15,innerbottommargin=15, outerlinewidth=4, linecolor=white, nobreak=true]
For all $q\in Q_{<1}$ and $a \in A$:
\begin{align*}
p_{(q,a)}&= \sum_{q' \in Q_{<1}} T(q,a)(q')p_q + \sum_{q' \in Q_{1}} T(q,a)(q') \\
p_q &\geq 0 \qquad\text{ and }\qquad p_q \geq p_{(q,a)}
\end{align*}
\hfill with the objective to \emph{Minimise} $\sum_{q\in Q_{<1}} p_q$.
\end{mdframed}
Note that we can achieve $p_q$ value when starting at $q$ by playing a safe strategy that at any state, $q'$, only uses an action, $a$, such that $p_{q'} = p_{(q',a)}$
and such that there is a positive probability to visit an accepting transition as done in (1).

\vspace{0.2em} \noindent\textbf{Compute Near-Optimal Strategies.} To find a strategy that is near-optimal with respect to the discounted reward without sacrificing the probability to satisfy the B\"{u}chi objective we do the following.
First, we remove all state-action pairs $(q,a)$ such that $p_q \neq p_{(q,a)}$ from $\mathcal P$, because taking
any such action would reduce the probability of generating an accepting run.
Next, we remove all states that have no actions left. 
Note that, in the example discussed in Lemma~\ref{lemma:noopt}, we would neither remove any states nor any state-action pairs.
Now, for the remaining states, $q \in Q'$, and state-actions pairs, $(q,a)$, we introduce variables $\rho_q$ and $\rho_{(q,a)}$, respectively. We find the optimal discounted reward values and a positional strategy $\tau$ that realises them using the following linear program.
\begin{mdframed}[backgroundcolor=black!10, roundcorner=10pt,leftmargin=1, rightmargin=1, innerleftmargin=15, innertopmargin=15,innerbottommargin=15, outerlinewidth=4, linecolor=white]
For all $q\in Q'$ and $a \in A(q)$:
\begin{align*}
\rho_q &\geq \rho_{(q,a)} \\
\rho_{(q,a)} &= 
\sum_{q' \in Q'} T(q,a)(q') \big(\rho(q,a,q') + \lambda \rho_{q'}\big)\ ,
\end{align*}
\hfill with the objective to \emph{Minimize} $\sum_{q\in Q'} \rho_q$.
\end{mdframed}

In the example shown in Lemma~\ref{lemma:noopt}, this would be to always produce $a$'s---which would not satisfy the B\"uchi objective.
However, any strategy obtained this way can either be followed long enough that the value of the tail is marginal, or we can initially follow this strategy, and switch to a strategy $\sigma^*$ that pursues the B\"uchi objective with a small probability in every step.
Both approaches will lead to a strategy that maximises the probability to satisfy the B\"uchi objective (in our example, with probability $1$) while providing an expected payoff that is near-optimal among the strategies that maximise satisfying the B\"uchi objective.

\subsection{Reinforcement Learning: Reward Translation}

When the MDP is not known, we use model-free RL to approximate optimal value and learn a near optimal strategy.
We present a reduction from the product MDP $\mathcal P = (Q, q_0, A, T,  \rho, F)$ to a related MDP $\mathcal P_\lambda$ such that optimal values from $\mathcal P_\lambda$ can be used to compute the optimal value in $\mathcal P$.

\begin{definition}[Reward Translation]
For a product MDP $\mathcal P = (Q, q_0, A, T,  \rho, F)$ with discount factor $\lambda$, consider a related MDP $\mathcal P_\lambda = (Q', (q_0,0), A, T',  \rho, F')$, where:
\begin{itemize}
    \item $Q' = (Q \times \set{0,1}) \cup \{t\}$ is the state space, where $t$ is a fresh trap state including the initial state 
    $(q_0, 0)$,
    \item $T'\colon Q' \times A \times Q' \to [0,1]$ is the transition function where 
    \begin{align*}
        T'((q, 0), a, (q', 0)) &= (1-\lambda) T(q, a, q') \\
        T'((q, 0), a, (q', 1)) &= \lambda T(q, a, q') \\
        T'((q, 1), a, (q', 1)) &= 
            \begin{cases}
                \zeta T(q, a, q') & \text{if $(q, a, q')  \in F$} \\
                T(q, a, q') & \text{otherwise}
            \end{cases}
        \\
        T'((q, 1), a, t) &= (1-\zeta) \sum\limits_{\{q'\mid (q,a,q')\in F\}}T(q,a,q')  \\
        T'(t, \sigma, t) &= 1
    \end{align*}
    where $\zeta \in (0,1)$ is a parameter,
    \item and $\rho'\colon Q' \times A \times Q' \to \mathbb{R}$ is the reward function where
    \begin{align*}
        \rho'((q, 0), a, (q', 0)) &= \rho(q, a,q') \\
        \rho'((q, 0), a, (q', 1)) &= \rho(q, a,q') \\
        \rho'((q, 1), a, (q', 1)) &= 0 \\
        \rho'((q, 1), a, t) &= f \\
        \rho'(t, a, t) &= 0
    \end{align*}
    where $f \geq 1$ is a (usually large) parameter.
\end{itemize}
\end{definition}

\begin{figure}[t!]
    \centering
    \begin{tikzpicture}[scale=0.85,transform shape]
    \colorlet{darkgreen}{green!40!black}
    
    \begin{scope}[name prefix=org-,xshift=-2cm,yshift=-2cm]
    \node[box state,fill=safecellcolor] (S0) {$q_0$};
    \node[prob state,fill=safecellcolor] (Prob)
    [below=1.5cm of S0] {};
    \node[box state,fill=safecellcolor] (S1)
    [left=2.5cm of Prob] {$q_1$};
    \node[box state,fill=safecellcolor] (S2)
    [right=2.5cm of Prob] {$q_2$};
    \path[->]
    (S2) edge[bend right=30, swap] node {$1,\rho_2$} (org-S0);
    \path[-]
    (S0) edge[swap] node[] {} (Prob);
    \path[->]
    (Prob) edge[swap] node[accepting dot, label={above:$p, \rho_0$}] {} (S1)
    edge node {$1-p, \rho'_0$} (S2)
    (S1) edge[bend left=30] node[accepting dot,label={$1,\rho_1$}] {} (org-S0);
    \end{scope}
    
    \node[single arrow, fill=purple!25] at (2.5cm,-3.5cm) {\phantom{m}};
    \begin{scope}[name prefix=disc-,xshift=8.5cm]
    \node[box state,fill=safecellcolor] (S0) {$(q_0,0)$};
    \node[prob state,fill=safecellcolor] (Prob)
    [below=1.5cm of S0] {};
    \node[box state,fill=safecellcolor] (S1)
    [left=3cm of Prob] {$(q_1,0)$};
    \node[box state,fill=safecellcolor] (S2)
    [right=3cm of Prob] {$(q_2,0)$};
    \path[->]
    (S2) edge[bend right=30, swap] node {$\lambda,\rho_2$} (disc-S0);
    \path[-]
    (S0) edge[swap] node[] {} (Prob);
    \path[->]
    (Prob) edge[swap] node{\small $p\lambda, \rho_0$} (S1)
    edge node {\small $(1-p)\lambda, \rho'_0$} (S2)
    (S1) edge[bend left=30] node{$\lambda,\rho_1$} (disc-S0);
    \end{scope}
    
    \begin{scope}[name prefix=trap-, xshift=8.5cm,yshift=-4cm]
    \node[box state,fill=safecellcolor] (S0) {$(q_0,1)$};
    \node[prob state,fill=safecellcolor] (Prob)
    [below=1.5cm of S0] {};
    \node[box state,fill=safecellcolor] (S1)
    [left=3cm of Prob] {$(q_1,1)$};
    \node[box state,fill=safecellcolor] (S2)
    [right=3cm of Prob] {$(q_2,1)$};
    \path[->]
    (S2) edge[bend right=-7, swap] node[pos=0.6, label={[label distance=-3mm]above:$1$}] {} (trap-S0);
    \node[prob state,fill=safecellcolor] (Tau1)
    [left=2cm of S1] {};
    \node[box state, darkgreen, thick, fill=goodcellcolor] (Sink)
    [below=1.7cm of Prob] {$t$};
    \path[-]
    (S0) edge[swap] node {} (Prob);
    \path[->]
    (Prob) edge[swap] node[below] {$p\zeta$} (S1)
    edge node[pos=0.34] {$1-p$} (S2)
    edge[darkgreen] node[black] {$p(1-\zeta)$, $f$} (Sink)
    (Tau1) edge[out=90,in=180] node {$\zeta$} (trap-S0)
    edge[out=270,in=180,swap,darkgreen] node[black] {$1-\zeta$, $f$} (Sink)
    (S1) edge[swap] node[below] {} (Tau1);
    \path[->]
    (disc-Prob) edge[bend left=-15] node[pos=0.15,label={[label distance=-4mm]right:\scriptsize $(1-p)(1-\lambda),\rho'_0$}] {} (trap-S2);
    \path[->]
    (disc-Prob) edge[bend right=-15] node[pos=0.1,label={[label distance=0mm]left:\scriptsize $p(1-\lambda),\rho_0$}] {} (trap-S1);
    \path[->]
    (disc-S1) edge[bend right=30] node[pos=0.3,label={left:$1-\lambda,\rho_1$}] {} (trap-S0);
    \path[->]
    (disc-S2) edge[bend left=30]  node[pos=0.2] {$1-\lambda,\rho_2$} (trap-S0);
    \end{scope}
    \end{tikzpicture}
    \caption{A translation of a product MDP with a combined B\"{u}chi-discounted reward objective with a discount factor $\lambda$ on the left hand side to an MDP with total reward objective on the right hand side. 
    Each transition is labelled with its probability followed by its reward. Rewards with value $0$ are omitted. Intuitively, the new MDP consists of two phases: the maximisation of the discounted reward (states labelled $(q_i,0)$) followed by maximisation of the B\"{u}chi objective (states labelled $(q_i,1)$). We move from the first phase to the second with probability $1-\lambda$ in each step. Any original accepting transition is marked with a green dot and is replaced in the new MDP with a transition that leads with probability $p(1-\zeta)$ to the new trap state $t$ and with probability $p\zeta$ to the original target state, where $p$ is the original transition probability.}
    \label{fig:add-sink}
\end{figure}
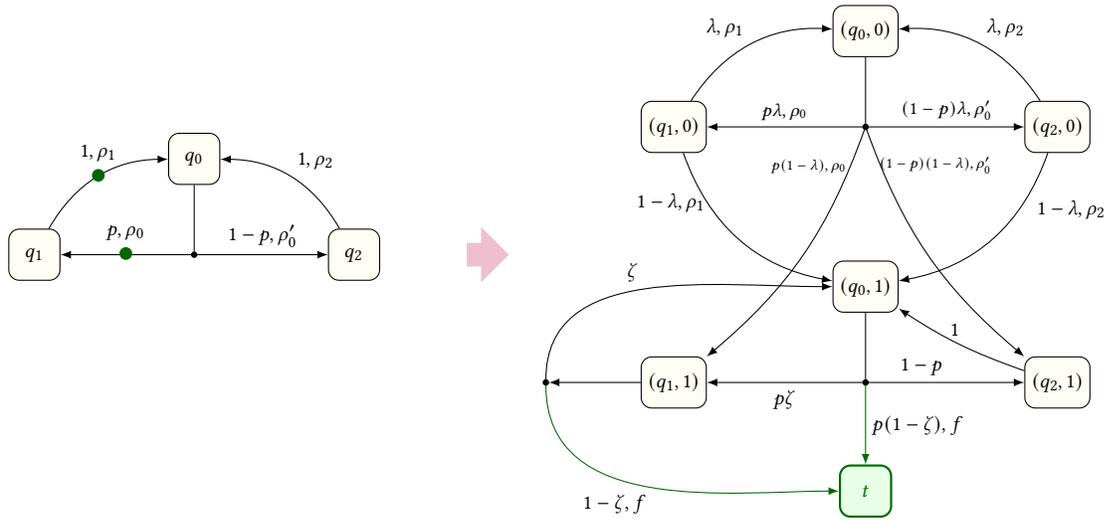

An example of this translation is shown in Figure~\ref{fig:add-sink}.
Note that while the MDP $\mathcal P_\lambda$ has two additional parameters, $f$ and $\zeta$, it only has numerical rewards and is therefore simpler to analyse. 

The parameter $\zeta$ has a similar meaning as in~\cite{Hahn19}: with probability $1-\zeta$ one declares that accepting edges will be seen forever after seeing a single accepting edge. 
Increasing the value of $\zeta$ means that the agent needs to see a larger number of accepting edges to obtain the same probability of declaring that accepting edges will be seen infinitely often. 
As in \cite{Hahn19}, if the B\"uchi objective is satisfied with probability $1$ any $0 < \zeta < 1$ is correct in our reduction. The parameter $f$ is used to weight the value of the B\"uchi objective relative to the discounted reward objective. As we will show, for large enough $f$, the two objectives are ordered lexicographically.
For $\mathcal P_\lambda$, we want to learn the optimal expected reward. For this, we define
    \begin{itemize}
        \item 
    $\val'(r) = \sum_{i=0}^\infty \rho'(r_i, a_i, r_{i+1})$ for any run $r$ and,
        \item 
    $\val_\sigma'\colon q \mapsto  \mathbb E \Set{ \val'(r) \mid r \in \Runs^\sigma_q(\mathcal P_\lambda)}$, and

        \item 
    
    $\val'\colon q \mapsto \sup_\sigma \val_\sigma'(q)$, as well as a strategy $\sigma$ with $\val_\sigma'(q_0)=\val'(q_0)$.
        \end{itemize}
We say that a positional strategy $\sigma$ is $\varepsilon$-consistent with value function $\val'$ on a set $S$ if $|\val'_\sigma (q) - \val'(q)| < \varepsilon$ for all $q \in S$.

Let us start with an established learning approach for B\"uchi objectives, which is a special case for the approach from \cite{Hahn19}.

\begin{ttheorem}[Limit Reachability~\cite{Hahn19}]
\label{theo:blearn}
For a product MDP $\mathcal P$ and a given parameter $f \geq 1$, there exists $\varepsilon>0$ and $\zeta^* \in (0,1)$ such that, for all $\zeta\in(\zeta^*,1)$ and all $q\in Q$, 
\[
\val'\big((q,1)\big) \in [f \cdot \blike(q),f\cdot \blike(q)+\varepsilon],
\]
and all positional strategies that are $\varepsilon$-consistent with $\val'$ from $Q\times{1}$ are optimal strategies for the B\"uchi objective.
\end{ttheorem}

\begin{ttheorem}
For a given product MDP $\mathcal P$ and $\varepsilon \in (0,1)$, there is a parameter $f^*\geq 1$
such that, for all $f\geq f^*$, there is a $\zeta^* \in (0,1)$ such that, for all $\zeta\in(\zeta^*,1)$, 
\[
\val'\big((q,0)\big) \in [\bval(q)+f\cdot\blike(q),\bval(q)+f\cdot\blike(q)+\varepsilon].\] 
\end{ttheorem}

\begin{proof}
We start with a number of simple observations.

\begin{enumerate}
    \item For every set of parameters,  $\mathcal P_\lambda$ is a total payoff MDP and therefore has positional optimal strategies.

    \item For every strategy that provides the optimal probability $\blike(q_0)$ to satisfy the B\"uchi objective, the following holds:
    for all $k\in \mathbb N$, under the assumption that the MDP moves to the $1$-copy after $k$ steps,
    the expected chance of satisfying the B\"uchi objective is $\blike(q_0)$.
    Moreover, this value is $\leq \blike(q_0)$ for every strategy, and, for some $k'\leq |Q|$ and all $k \geq k'$, it is $< \blike(q_0)$ for a non-optimal positional strategy. 
    Thus, for any positional non-optimal strategy, there is a probability $p < \blike(q_0)$ such that the chance of meeting the B\"uchi objective is $\leq p$.
    Let $p^*$ be the maximal such value. (Note that the set of positional strategies is finite.)
    
\item Let $d$ be the maximal difference in expected discounted payoff between two strategies. ($d$ can be bounded by maximum value of $\rho$ divided by $1-\lambda$).
Then we choose
$f^* \geq \frac{d}{\gamma(\blike(q_0)-p^*)}+1$. 
\end{enumerate}
With these observations, we choose $f^*$ as in (3) and, for $f\geq f^*$, $\zeta^*$ as in Theorem \ref{theo:blearn}.
This provides correct strategies for obtaining the B\"uchi objective with maximal probability from all positions in the $1$-copy, and values that approximate $f\cdot \blike\big((q,1)\big)$ with $\varepsilon$ precision.

Let us now consider two positional strategies for $\mathcal R_\gamma$: $\sigma$, which maximises the chance of satisfying the B\"uchi objective, and $\tau$, which does not.
By (3) and Theorem \ref{theo:blearn}, this implies that the expected reward of $\sigma$ is at least $1-\varepsilon$ better than the expected reward of $\tau$.
Thus, a positional strategy with optimal reward will also maximize the chance of satisfying the B\"uchi objective.

Let us now consider two positional strategies for $\mathcal P_\lambda$ that are both $\varepsilon$-consistent with $\val$ on $Q \times \{1\}$ (and thus optimal w.r.t.\ the B\"uchi objective there), $\sigma$ and $\tau$, such that the expected reward in the $0$-copy of $\sigma$ is at least $\varepsilon$ higher than that of $\tau$.
Then, using Theorem \ref{theo:blearn}, the expected reward of $\sigma$ in $\mathcal P_\lambda$ is higher than the reward by $\tau$.

Putting these two together, we get that the optimal solution for $\mathcal R_\gamma$ provides that, for all $k \in \mathbb N$, the probability of satisfying the B\"uchi objective after $k$ steps is $\blike(q_0)$ while, among the strategies with this property, the expected discounted reward is $\varepsilon$-optimal.
\end{proof}

\vspace{0.2em}\noindent\textbf{Learning.}
For given parameters, $\mathcal P_\lambda$ is simply an MDP with total rewards \emph{and} contraction on $Q \times \{0\}$ and a reachability objective in $Q \times \{1\} \cup \{t\}$.
The values and strategies to obtain them can be learned with standard techniques, such as $Q$-learning.
Note that the parameter $\varepsilon$ can be selected after learning is complete since the strategy for $\mathcal P_\lambda$ is independent of $\varepsilon$.

\begin{ttheorem}
Given $\mathcal P_\lambda$ and parameters $f$ and $\zeta$, we can use $Q$-learning to find $\val'$ and an optimal strategy $\sigma$.
\end{ttheorem}

These three theorems provide us with a way to robustly infer a near-optimal strategy for $\mathcal M$ from $\mathcal R$ for  appropriate parameters $f$ and $\zeta$:
To obtain a $2\varepsilon$-optimal strategy for $\mathcal M$, we can simply find or approximate an optimal strategy for $\mathcal R$.
To transfer this strategy to $\mathcal M$, we can follow the $0$-copy  ($Q\times \{0\}$) long enough; say $k$ steps, so that the contribution of all but the first $k$ transitions is smaller than $\varepsilon$.
We can then move on to follow the strategy for the $1$-copy ($Q\times \{1\}$).

\begin{corollary}
For a product MDP $\mathcal P$, a $\varepsilon>0$ and appropriate parameter $f\geq 1$ and $\zeta \in (0,1)$, we can infer a near optimal strategy for $\mathcal P$ from a near optimal strategy for $\mathcal P_\lambda$ with parameters $f$ and $\zeta$.
\end{corollary}

\section{Experimental Results}
\label{sec:experiments}

\begin{table}[t]
    \centering
    \begin{tabular}{p{5cm}rrrrrrrrrrrrr}
        \toprule
        \belowrulesepcolor{black!10} 
        Name  & states & time (s) & $f$ & $\zeta$ & $\gamma$ & $\alpha$ & $\epsilon$ & init & ep-l & ep-n \\
        \aboverulesepcolor{black!10} 
        \midrule
        cheapest & 6 & 0.56 &  &  &   &  & 0.50 & &  & 3k \\
        promises & 16 & 0.71 &  &  &  &  & 0.50 & &  & 3k \\
        robot 4x4 & 64 & 0.30 &  &  &   &  & 0.20 & &  & 1k \\
        virus & 859 & 61.28 &  & 0.95 &  &  &  & & 50 & 150k \\
        busyRing2 & 169 & 45.38 &  &  &   & 0.005 & 0.75 & &  & 175k \\
        twoWECs & 6 & 4.80 & 500 &  &  & 0.005 & 0.20 &  &  & 30k \\
        officeZapPatrol & 876 & 165.05 &  & 0.5 &  &  &  & 7 & 100 & 350k \\
        officePreferences & 1248 & 77.27 & 100 & 0 &  & 0.002 &  & 100 & 150 & 200k \\
        officeZapPreferences & 1254 & 111.01 & 100 & 0 &  & 0.005 & 0.01 & 100 & 1500 & 250k \\
        \bottomrule
    \end{tabular}
    \caption{Q-learning results. Blank entries indicate that default parameters are used. The default parameters are $f = 10$, $\zeta = 0.99$, $\gamma = 0.999$, $\alpha = 0.01$, $\epsilon = 0.1$, $\text{init}=0$, $\text{ep-l}=20$, and $\text{ep-n}=20k$. The same external discount factor of $\lambda=0.99$ was used for all experiments.}
    \label{tab:results}
\end{table}
We implemented the construction described in Section~\ref{sec:omega_rm} in the tool \textsc{Mungojerrie}~\cite{Mungojerrie}.
The construction is implemented on-the-fly, where the states of the MDP and the $\omega$-regular reward machine are kept independent and concatenated at each time step. We ran tabular Q-learning on multiple case studies, as seen in Table~\ref{tab:results}. Table~\ref{tab:results} shows the name of the case study, the number of states in the product MDP, the time taken for learning in seconds, the value of $f$, $\zeta$, and algorithm discount factor $\gamma$. The table also shows the learning rate $\alpha$, the $\epsilon$-greedy exploration rate $\epsilon$, the value the Q-table was initialised to, the episode reset length, and the number of training episodes. The episode reset length is the number of time steps between accepting edges in the second layer that needs to be exceeded for the episode to be reset. After learning is complete, we verify that the values of the learned strategy matches the values computed by the linear programs described in Section~\ref{sec:omega_rm}. This ensures that the learned strategy can be transformed into an $\varepsilon$-optimal strategy for any $\varepsilon > 0$. For all of our experiments, we use the same external discount factor $\lambda = 0.99$.

Example \emph{cheapest} shows how quantitative rewards may supplement an $\omega$-regular specification to steer the agent toward a path of minimum \emph{cost} rather than a path of minimum \emph{length}.  Example \emph{promises} illustrates the use of quantitative rewards to model advances that an agent may get in return for a promise to fulfil an obligation encoded as an $\omega$-regular objective.  Examples \emph{robot 4x4} and \emph{virus} were used in \cite{Hahn21c} in the context of multiple $\omega$-regular properties lexicographically combined.
Their use here demonstrates the wider set of specifications allowed by $\omega$-regular reward machines.  In \emph{busyRing2}, quantitative rewards are used to measure the fraction of time spent in negotiating the asynchronous arbitration protocol.  In \emph{twoWECs}, quantitative rewards indicate preference between sets of states that satisfy the $\omega$-regular specification.  The last three examples are based on the office grid-world discussed in the introduction. Since these examples require deep exploration in order to discover the near-optimal strategy, we initialised the Q-table optimistically to aid exploration.

Overall, our experimental results demonstrate the effectiveness and versatility of $\omega$-regular reward machines in solving a wide range of RL problems with complex specifications. By leveraging the power of $\omega$-regular objectives and discounted rewards, we can specify and learn complex behaviours and preferences that would be difficult to express with traditional RL techniques. Moreover, our open-source implementation may allow researchers and practitioners to easily apply these techniques to their own problems.

\section{Conclusion}
\label{sec:conclusion} 
Reward machines and formal specifications are two leading reward programming languages that roughly correspond to imperative (how to give rewards?) and declarative (what to give rewards for?) ways of expressing programmer's intent. 
This paper presents a hybrid approach that combines the declarative and imperative specifications.
The $\omega$-regular RMs can be constructed by expressing declarative specification in LTL and the imperative reward specification as reward machines. 
The $\omega$-RMs strictly generalise both reward machines (from finite-horizon behaviour to infinite-horizon) and $\omega$-regular languages (from qualitative to quantitative satisfaction) in a natural framework. 
We presented a parametric reduction from the optimisation problem over $\omega$-RM to an optimal reward reachability problem that can be constructed and learned in a model-free manner. 

\begin{acks}
This work was supported in part by the EPSRC through grants EP/X017796/1 and EP/X03688X/1, the NSF through grant CCF-2009022 and the NSF CAREER award CCF-2146563; and the 
EU's Horizon 2020 research and innovation programme \includegraphics[height=8pt]{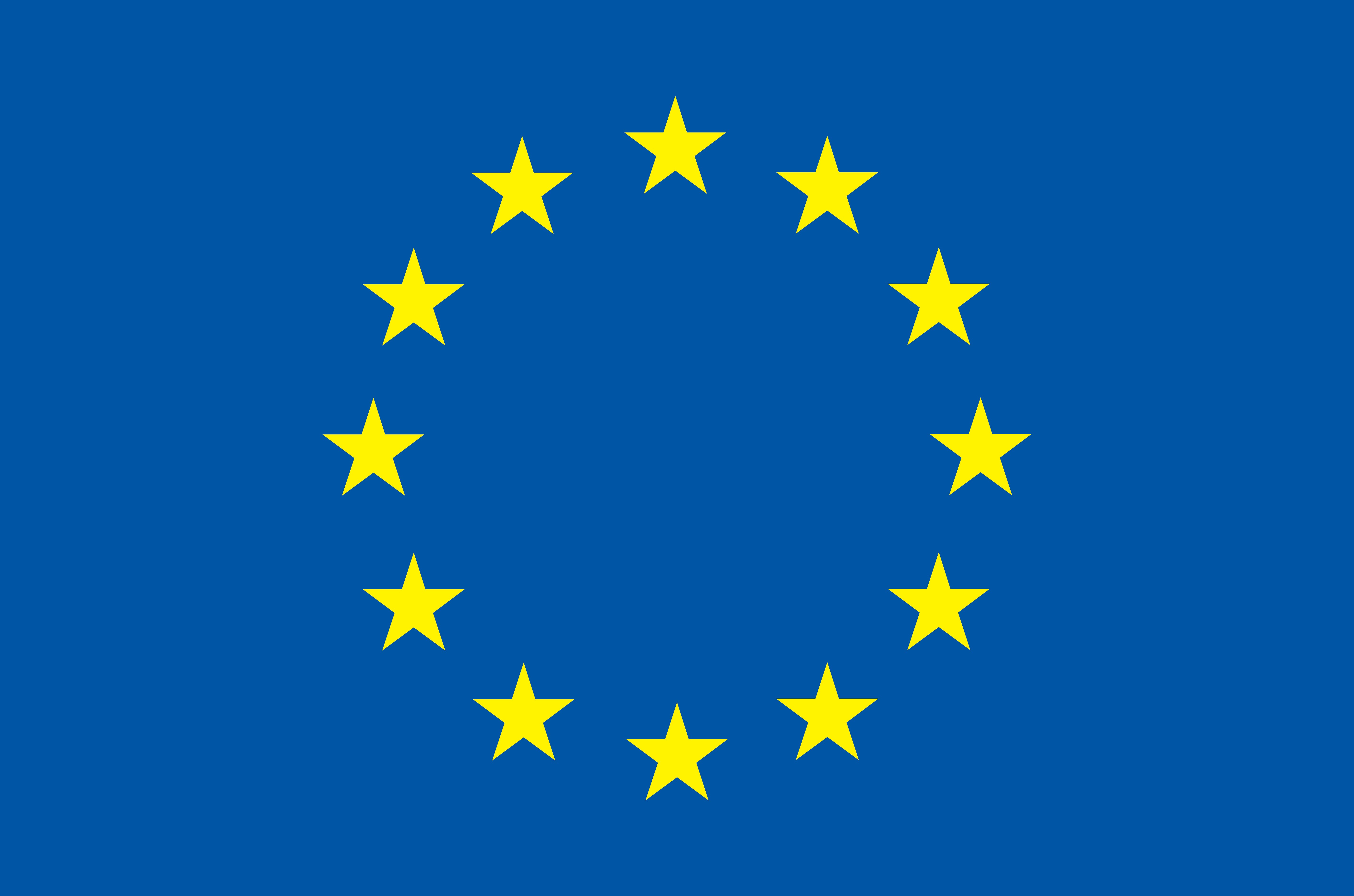}  under grant agreements No 864075 (CAESAR).
\end{acks}

\bibliographystyle{ACM-Reference-Format}
\bibliography{papers}

\end{document}